\documentclass[journal]{IEEEtran}
\usepackage{array}
\usepackage{lineno}
\usepackage{graphicx}
\usepackage{epstopdf}
\usepackage{xcolor}
\usepackage{amssymb,amsfonts,amsmath,amscd,mathrsfs,bm}
\usepackage{fancyhdr}
\usepackage{marvosym}
\usepackage{amsthm}
\usepackage{amsmath}
\usepackage{bm}
\usepackage{booktabs}
\usepackage{longtable}
\usepackage{multirow}
\usepackage{makecell}
\usepackage{hhline}
\usepackage{stmaryrd}
\usepackage{color, xcolor}
\usepackage{dcolumn}
\usepackage{algorithm}
\usepackage{algorithmic}
\usepackage{lipsum}
\usepackage[mathscr]{eucal}
\usepackage{txfonts}
\usepackage{mathtools}
\usepackage{booktabs}
\usepackage{cite}
\usepackage{subfigure}
\usepackage{balance}
\usepackage{float}
\usepackage{hyperref}
\usepackage{gensymb}
\usepackage{verbatim}
\usepackage{setspace}
\usepackage{ulem}
\makeatletter 

\newcommand{\Rmnum}[1]{\expandafter\@slowromancap\romannumeral #1@}
\makeatother

\begin{document}
\newcommand{\tabincell}[2]{\begin{tabular}{@{}#1@{}}#2\end{tabular}}

\title{Solving PDEs with  Unmeasurable Source Terms Using Coupled Physics-Informed Neural Network with Recurrent Prediction for Soft Sensors}

\author{Aina Wang, 
        Pan Qin, 
        Xi-Ming Sun,~\IEEEmembership{Senior Member,~IEEE}, 

\thanks{The authors are with the Key Laboratory of Intelligent Control and Optimization for Industrial Equipment of Ministry of Education and the School of Control Science and Engineering, Dalian University of Technology, Dalian  116024, China  e-mail: WangAn@mail.dlut.edu.cn, qp112cn@dlut.edu.cn, sunxm@dlut.edu.cn $\left(\textit{Corresponding author: Pan Qin}\right)$}}

\markboth{}
{Shell \MakeLowercase{\textit{et al.}}: Bare Demo of IEEEtran.cls for IEEE Journals}

\maketitle

\begin{abstract}
Partial differential equations (PDEs) are a  model  candidate for soft sensors  in industrial processes with  spatiotemporal dependence. Although physics-informed neural networks (PINNs) are a promising   machine learning method for solving PDEs, they are  infeasible for the 
nonhomogeneous PDEs with unmeasurable source terms.
To this end,  a coupled PINN (CPINN) with a recurrent prediction (RP) learning strategy (CPINN-RP) is proposed. 
First,  CPINN composed of $NetU$ and $NetG$ is proposed.
$NetU$ is for approximating   PDEs solutions  and
$NetG$ is for  regularizing the training of  $NetU$.
The two networks are integrated into a data-physics-hybrid loss function.
Then, we theoretically prove that the proposed CPINN has a satisfying approximation 
capability  for solutions to nonhomogeneous PDEs  with unmeasurable source terms.
Besides the theoretical aspects, we propose a hierarchical training strategy to optimize and couple $NetU$ and $NetG$.
Secondly, $NetU$-$RP$ is proposed  for compensating  information loss in data sampling to improve the prediction performance, in which RP is the recurrently delayed outputs of well-trained CPINN and hard sensors.
Finally, the artificial and practical datasets are used to verify the feasibility and effectiveness of CPINN-RP for  soft sensors.
\end{abstract}

\IEEEpeerreviewmaketitle

\section{Introduction}\label{se1}

\IEEEPARstart
{I}{ndustrial} processes have numerous  variables   measured with hard sensors, such as temperature and displacement sensors for rotating machinery.
However, the limitations of  operating environment challenge the equipment of hard sensors \cite{lui2022supervised}.
To tackle the problems, soft sensors are widely used to estimate the key variables using
easy-to-measure  variables and mathematical models \cite{9929274},\cite{ 9673119}. Widely used models and algorithms in soft sensors are mainly for time-series data, such as  the Kalman filter\cite{welch1995introduction} and observer-based methods\cite{fischer2022fault}.
However, as far as spatiotemporal  variables  are involved,  these methods cannot handle the spatial information. 
Partial differential equations (PDEs), such as the parabolic-type-PDE  heat equation and the hyperbolic-type-PDE wave equation \cite{tang2023adaptive, 8948299},  are one of the  mathematical models for describing  spatiotemporal dependence in engineering~\cite{9868162},
physics\cite{1514663}, medicine\cite{oszkinat2022uncertainty},
finance\cite{sirignano2018dgm},
and weather forecasting\cite{kashinath2021physics}.
Thus,  PDEs are a  model candidate for soft sensors in industrial processes with spatiotemporal variables \cite{9062588}.
Accordingly, the methods for solving PDEs play a vital role in soft sensors.

 In some cases, analytical methods for solving PDEs may be difficult to achieve caused of the  spatiotemporal coupling nature and unknown dynamics \cite{8511068}.
Numerical approaches, such as the finite difference method (FDM) \cite{smith1985numerical} and the finite element method (FEM)~\cite{li2017numerical,dziuk2013finite}, have been widely used for solving PDEs.
FDM uses a topologically square lines network to construct a discretization scheme for solving PDEs.
However, complex geometries in multiple dimensions challenge FDM~\cite{peiro2005finite}.
On the other hand, complex geometries can be treated with FEM~\cite{reddy2019introduction}.
The  difficulty of classical numerical approaches is  the tradeoff between the accuracy and efficiency caused by forming meshes.

Among numerical methods, the famous Galerkin method uses the linear combination of basis functions to approximate the PDEs solutions \cite{zhuang2010aspects}.  
 As a kind of variant of the  Galerkin method, several works  replaced the linear combination of basis functions with machine learning models  to construct data-efficient methods for solving PDEs~\cite{sirignano2018dgm}.  Successful applications of deep learning methods in various fields, such as image processing \cite{ye2021deep},
text mining~\cite{9260228}, and speech recognition~\cite{9927496}, ensure that they are excellent replacers of  the linear combination of basis functions in the  Galerkin method.
Consequently, using the well-known approximation capability of neural networks to solve PDEs is a natural idea and has been investigated  previously~\cite{meade1994numerical, lagaris1998artificial, lagaris2000neural}.
The  physics-informed neural networks (PINNs) were introduced by Raissi et al.\cite{raissi2019physics} to solve  PDEs with given initial and boundary conditions and priori physical knowledge\cite{9786547, cuomo2022scientific,zobeiry2021physics, chen2021physics}.
Note that solving PDEs by machine learning methods is usually meshfree.

PDEs can be classified into 
homogeneous and nonhomogeneous types.
The homogeneous PDEs can describe the industrial processes without  source terms.
The nonhomogeneous PDEs can reveal the continuous energy propagation behaviors of  sources describing the industrial processes  driven by sources.
The abovementioned methods  have achieved successful applications for solving homogeneous PDEs. Other recent developments for solving nonhomogeneous PDEs have also been investigated.
The functional forms of the solution and the source term were both assumed to be unknown in~\cite{yang2022multi},  in which the measurements of the source term can be obtained separately from the measurements of the solution.
The recent work\cite{gao2022physics} 
proposed a graph neural network  to solve the PDEs with unmeasurable source terms, 
where the source terms were assumed to be constant. 
Although the aforementioned methods have made great progress in solving nonhomogeneous PDEs, independent measurements of the source terms in  the operating domain cannot always be easily obtained from practical situations, such as the heat sources of  engines\cite{guzzella2009introduction}.
Furthermore, the existing methods with the assumption of the constant source terms cannot be extended to investigate the industrial processes with spatiotemporal dependence.
Thus, solving PDEs following with unmeasurable  source terms is an under-investigated issue.

To solve PDEs with unmeasurable  source terms, this article proposes a coupled PINN (CPINN) with a recurrent prediction (RP) learning strategy  (CPINN-RP), which can be considered  a two-phase soft sensor.
In the first phase, CPINN composed of $NetU$ and $NetG$ is proposed. 
$NetU$ is for approximating  solutions to PDEs  and $NetG$ is for regularizing the training of $NetU$.
The two networks are integrated into a data-physics-hybrid loss function.
The approximation capability of CPINN is theoretically proved using the second power of $L^2$-norm in this article, under the assumption of  unmeasurable source terms.
Note that the recent work\cite{stenen2022combining} also theoretically claimed the approximation capability of neural networks. However, our theoretical results give the bound of the approximation error of CPINN, which has not been considered in \cite{stenen2022combining}. 
Besides the theoretical aspects, we also propose a hierarchical training strategy to optimize and couple $NetU$ and $NetG$  for practical situations.
If PDEs include the partial derivative of the temporal variables, the training dataset will be sampled with a fixed temporal interval. This discretization strategy can cause information loss.
To this end, $NetU$-$RP$ is proposed in the second phase, which is for compensating  information loss  to improve the prediction performance. RP is the recurrently delayed outputs of well-trained CPINN and hard sensors.
Finally,   the artificial and practical datasets are used to verify the feasibility and effectiveness of  CPINN-RP for  soft sensors in industrial processes with spatiotemporal dependence and unmeasurable  sources.

The rest of this article is organized as follows.
The classical PINNs are briefly reviewed in Section \uppercase\expandafter{\romannumeral2}.
CPINN-RP  is proposed in Section \uppercase\expandafter{\romannumeral3}.
In Section {\uppercase\expandafter{\romannumeral4}},  CPINN-RP is verified with  the artificial and practical datasets.
Finally, Section \uppercase\expandafter{\romannumeral5} concludes this article.
\section{Brief Review of PINNs}
We   briefly review the basic idea of PINNs, in which the following homogeneous PDEs are considered:\begin{equation}\label{eq:forward1}
u_t(\boldsymbol{x}, t)+\mathcal{N}[u(\boldsymbol{x}, t)]=0, \ \boldsymbol{x} \in \Omega \subseteq \mathbb{R}^{d}  ,\ t \in[0, T] \subset \mathbb{R}.
\end{equation}
Here, $\boldsymbol{x}$ is the spatial variable;
$t$ is the temporal variable;
$u: \mathbb{R}^{d}\times\mathbb{R}\rightarrow \mathbb{R}$ denotes the  solution;
$\mathcal{N}[\cdot ]$ is a series of  differential operators;
the compact domain $\Omega \subseteq \mathbb{R}^{d}$ is a spatial bounded open set with the boundary $\partial\Omega$.

PINNs  can be  trained by minimizing the loss function
\begin{equation} \label{eq:loss}
{{\rm {MSE}}_H} = {\rm {MSE}}_{DH}+ {\rm {MSE}}_{PH}.
\end{equation}
Here, ${\rm MSE}_{DH}$ is formulated as the following
\begin{equation}\label{eq:20221115-1}
{\rm MSE}_{DH}= \displaystyle \frac{1}{{\rm card}\left(D\right)}\underset{{(\boldsymbol{x}, t,u)\in D}}{\sum} \left( \hat u\left(\boldsymbol{x},t;\boldsymbol{\Theta}_{U}\right) -{u\left(\boldsymbol{x},t\right)}\right)^{2},
\end{equation}
where $D$ is the training dataset and  ${\rm card}\left(D\right)$ is the cardinality of $D$. 
$\hat u\left(\boldsymbol{x},t;\boldsymbol{\Theta}_{U}\right)$ is the function of PINNs  to approximate the solution satisfying  \eqref{eq:forward1}, with $\boldsymbol{\Theta}_{U}$ being a set of parameters.
This mean-squared-error  term
\eqref{eq:20221115-1} is considered a data-driven loss.

The left-hand side of \eqref{eq:forward1} is used to define a residual function as the following
\begin{equation}\label{eq:forward2}
f(\boldsymbol{x},t) = u_t(\boldsymbol{x}, t)+\mathcal{N}[u(\boldsymbol{x}, t)].
\end{equation}
Consequently, ${\rm {MSE}}_{PH}$ is formulated as 
\begin{equation}\label{eq:20221111-qp1}
{\rm MSE}_{PH}= \displaystyle \frac{1}{{\rm card}\left(E\right)} \underset{(\boldsymbol{x}, t)\in {E}}{\sum}\hat f\left(\boldsymbol{x}, t;\boldsymbol{\Theta}_{U}\right)^2,
\end{equation}
where $E$ denotes the set of collocation points, $\hat f\left(\boldsymbol{x}, t;\boldsymbol{\Theta}_{U}\right)$ is an estimation to the residual function $f(\boldsymbol{x},t)$ based on $\hat{u}\left(\boldsymbol{x},t;\boldsymbol{\Theta}_{U}\right) $.
$\hat f\left(\boldsymbol{x}, t;\boldsymbol{\Theta}_{U}\right)$ is obtained as the following
\begin{equation}\label{eq:f}
\hat{f}\left(\boldsymbol{x}, t;\boldsymbol{\Theta}_{U}\right) = \hat{u}_t\left(\boldsymbol{x},t;\boldsymbol{\Theta}_{U}\right) +\mathcal{N}\left[\hat{u}\left(\boldsymbol{x},t;\boldsymbol{\Theta}_{U}\right) \right],
\end{equation}
where $ \hat{u}_t\left(\boldsymbol{x},t;\boldsymbol{\Theta}_{U}\right)$ and $\mathcal{N}\left[\hat{u}\left(\boldsymbol{x},t;\boldsymbol{\Theta}_{U}\right) \right]$ are obtained  using automatic differentiation (AD)~\cite{baydin2018automatic}. 
MSE$_{PH}$ is used to regularize $\hat u\left(\boldsymbol{x},t;\boldsymbol{\Theta}_{U}\right) $ to satisfy  \eqref{eq:forward1}, 
which is considered a physics-informed loss.
Readers are referred to \cite{raissi2019physics} for the details.

\vspace{-0.2cm}
\section{Methods}
\subsection{CPINN for Solving Nonhomogeneous PDEs}
The nonhomogeneous PDEs under study are of  the following generalized form
\begin{equation}\label{eq:PaperGeneral1}
u_t(\boldsymbol{x},t)+\mathcal{N}[u(\boldsymbol{x},t)]=g(\boldsymbol{x},t),\hspace{0.5mm}
\boldsymbol{x} \in \Omega ,\hspace{0.5mm} t \in[0, T],
\end{equation}
where $\boldsymbol{x}$, $t$, $\Omega$,  {$u$}, and $\mathcal{N}[\cdot ]$ are similar to \eqref{eq:forward1};
$g: \mathbb{R}^{d}\times \mathbb{R} \rightarrow \mathbb{R} $ is the source term active in $\Omega$ and cannot always be easily measured with hard sensors, owing to the practical operating environment limitations.

The residual function is defined for the nonhomogeneous case as the following
\begin{equation}\label{eq:PaperGeneral2}
f_{N}(\boldsymbol{x},t)=u_t(\boldsymbol{x},t)+\mathcal{N}[u(\boldsymbol{x},t)]-g(\boldsymbol{x},t).
\end{equation}
When $g(\boldsymbol{x},t)$ is measurable, $f_N (\boldsymbol{x},t)$ is obtained using  AD with respect to  \eqref{eq:PaperGeneral2}.
However, the unmeasurable $g(\boldsymbol{x},t)$ will lead to  the aforementioned regularization {\eqref{eq:20221111-qp1}} infeasible. 
To this end, CPINN is first proposed to approximate the solutions to PDEs with unmeasurable source terms in  \eqref{eq:PaperGeneral1}. 
The proposed CPINN  contains two neural networks:
1) $NetU$ is for approximating the solution satisfying \eqref{eq:PaperGeneral1};
2) $NetG$ is for regularizing the training of $NetU$.

The training dataset {$D$} is uniformly sampled from the industrial processes governed by \eqref{eq:PaperGeneral1} using available hard sensors. 
 $D$ is divided into $D = D_{B}\cup D_{I}$ with $D_{B}\cap D_{I}=\varnothing$,
where $D_{B}$ denotes the training dataset sampled from the  initial and boundary conditions and $D_{I}$ is the training dataset sampled from the interior of $\Omega$.
The collocation points $E=E_B\cup E_I$, where $\left(\boldsymbol{x},t\right)\in E_B$ and  $\left(\boldsymbol{x},t\right)\in E_I$ correspond to those of $\left(\boldsymbol{x},t,u\right)\in D_{B}$ and $\left(\boldsymbol{x},t,u\right)\in D_{I}$, respectively.
Then,
we use the following data-physics-hybrid loss function
\begin{equation} \label{eq:loss}
{\rm {MSE}}_N = {\rm {MSE}}_{DN}+ {\rm MSE}_{PN}
\end{equation}
to train CPINN. ${\rm {MSE}}_{DN}$ and ${\rm {MSE}}_{PN}$ in \eqref{eq:loss} are the data-driven loss and physics-informed loss for the nonhomogeneous PDEs, respectively.
Let 
$$\hat e_N\left(\boldsymbol{x},t;\boldsymbol{\Theta}_{U}\right)= \hat u\left(\boldsymbol{x},t;\boldsymbol{\Theta}_{U}\right) -{u\left(\boldsymbol{x},t\right)} $$
denotes a data approximation error, i.e., CPINN approximation error.
Here, $\hat u\left(\boldsymbol{x}, t;\boldsymbol{\Theta}_{U}\right)$ is the function of $NetU$, with $\boldsymbol{\Theta}_{U}$ being  a set of parameters.
${\rm MSE}_{DN}$ is as the following form
\begin{equation}\label{eq:MSE_DN}
{\rm MSE}_{DN}=
\hspace{-1mm} \displaystyle \frac{1}{{\rm card}\left( D\right)}\underset{	{\left(\boldsymbol{x}, t,u\right)\in D }}{\sum} \hat e_N\left(\boldsymbol{x},t;\boldsymbol{\Theta}_{U}\right)^{2}.
\end{equation} 
${\rm MSE}_{PN}
$ is as the following
\begin{equation}\label{eq:cost function MSP_EN}
\begin{aligned}
{\rm MSE}_{PN}
= \displaystyle \frac{1}{{\rm card}\left(E\right)}\underset{(\boldsymbol{x}, t)\in {E}}{\sum}\hat{f}_N\left(\boldsymbol{x},t;\boldsymbol{\Theta}_{U}\right)^{2},
\end{aligned}
\end{equation}
where $\hat f_N\left(\boldsymbol{x}, t;\boldsymbol{\Theta}_{U}\right)$ denotes a physics-informed approximation
error, an estimation to  the residual function ${f}_N(\boldsymbol{x},t)$ based on $\hat{u}\left(\boldsymbol{x},t;\boldsymbol{\Theta}_{U}\right) $.
$\hat f_N\left(\boldsymbol{x}, t;\boldsymbol{\Theta}_{U}\right)$ is obtained as the following
\begin{equation}\label{eq:20221111-qp-2}
\hat{f}_N\left(\boldsymbol{x}, t;\boldsymbol{\Theta}_{U}\right) =\hat f\left(\boldsymbol{x}, t;\boldsymbol{\Theta}_{U}\right)- g\left(\boldsymbol{x}, t\right),
\end{equation}
where 
$\hat f(\boldsymbol{x}, t;\boldsymbol{\Theta}_{U})$ has been defined by  \eqref{eq:f}.
${\rm MSE}_{PN}$ is used to
regularize $\hat{u}\left(\boldsymbol{x},t;\boldsymbol{\Theta}_{U}\right) $ to satisfy \eqref{eq:PaperGeneral1}.
Consequently, when CPINN is used to approximate the solutions to \eqref{eq:PaperGeneral1},  the following equation is considered:
\begin{equation}
\begin{aligned}
\label{eq:appriximation error}
&{\hat u}_t\left(\boldsymbol{x},t;\boldsymbol{\Theta}_{U}\right)+\mathcal{N}\left[{\hat u}\left(\boldsymbol{x},t;\boldsymbol{\Theta}_{U}\right)\right]=\\
&\hspace{5mm}{\hat g}\left(\boldsymbol{x},t;\boldsymbol{\Theta}_{G}\right)+{\hat g}_d\left(\boldsymbol{x},t;\boldsymbol{\Theta}_{G}\right)+{\hat f}_N\left(\boldsymbol{x},t;\boldsymbol{\Theta}_{U}\right).
\end{aligned}
\end{equation}
Here, $\hat g\left(\boldsymbol{x}, t;\boldsymbol{\Theta}_{G}\right)$ is the function of $NetG$,  with $\boldsymbol{\Theta}_{G}$ being a set of parameters;
${\hat g}_d\left(\boldsymbol{x},t;\boldsymbol{\Theta}_{G}\right) = g\left(\boldsymbol{x}, t\right) -\hat g\left(\boldsymbol{x}, t;\boldsymbol{\Theta}_{G}\right)$
 is the approximation error of $NetG$.
In fact, the approximation $\hat{u}$ obtained by CPINN with respect to exact solution $u$ satisfies PDE \eqref{eq:appriximation error} obtained by perturbing \eqref{eq:PaperGeneral1} with  ${\hat e}_N$, ${\hat f}_N$,  and ${\hat g}_{d}$.
\subsection{Approximation Theorem for CPINN}
In this section, the approximation capability of CPINN for the solutions to PDEs following with  unmeasurable source terms  is theoretically proved based on the second power of  $L^2$-norm.
The definition of  well-posed PDE is first given as the following:
\newtheorem{Definition}{\bf Definition}
\begin{Definition}[Well-posed PDE]\label{defi1}
PDE in \eqref {eq:PaperGeneral1} is called well-posed, if the following two conditions are satisfied: 1) There exists a unique solution for all functions $g\left(\boldsymbol{x},t\right)$   in \eqref{eq:PaperGeneral1};  2) For each function $g_1$ and $g_{2}$ satisfying \eqref{eq:PaperGeneral1}, the corresponding solutions $u_1$ and $u_2$ satisfy
\begin{equation}\label{defiLp}
\left\|u_1-u_2\right\| \leq c\left\|g_1-g_2\right\|
\end{equation}
for a fixed and finite constant $0<c \in \mathbb{R}$.  $c$ is referred to as the Lipschitz constant of PDEs.
Here, $\|\cdot\|$ denotes the $L^1$-norm. 
\end{Definition}

To discuss the general approximation capability of CPINN, instead of the mean squared error in \eqref{eq:MSE_DN} and \eqref{eq:cost function MSP_EN}, the second power of $L^2$-norm  is used  as the following:
\begin{equation}\label{eq:lossbound}
\begin{aligned}
L_{DN}=\int_{0}^{T}\int_{\bar\Omega} {\hat e}_N\left(\boldsymbol{x},t;\boldsymbol{\Theta}_{U}\right)^2
 {\rm d} \boldsymbol{x}{\rm d } t.
 \end{aligned}
 \end{equation}
and
 \begin{equation}\label{eq:lossin}
L_{PN}=\int_{0}^{T}\int_{\bar\Omega} {\hat f}_N\left(\boldsymbol{x},t;\boldsymbol{\Theta}_{U}\right)^2
 {\rm d} \boldsymbol{x}{\rm d } t.
\end{equation}
Here, let $\bar\Omega=\Omega\cup\partial\Omega$ denote the closure of $\Omega$;
$\hat{e}_N$ and $\hat{f}_N$  are assumed to have finite $L^2$-norm. 
Consequently, the loss function for   training CPINN is given as the following:
\begin{equation}
\label{eq:lossU}
L_{\hat U}= \left[T{\rm card}\left(\bar\Omega\right)\right]^{-1}\left(L_{DN}+L_{PN}\right).
\end{equation}The following theorem guarantees the approximation capability of CPINN for an exact solution satisfying \eqref {eq:PaperGeneral1}.
\newtheorem{theorem}{\bf Theorem}
\begin{theorem}\label{thm1}
Assume PDE  in \eqref{eq:PaperGeneral1} to be well-posed. 
Then, for all $\varepsilon>0$, there exists a $\delta>0$,
\begin{equation}\notag
\label{eq:loss theorem}
L_{\hat U}<\delta \Longrightarrow\left\|\hat{u}-u\right\|<\varepsilon .
\end{equation}	
\end{theorem} 
\begin{proof}
According to \eqref{eq:appriximation error} and Definition \ref{defi1},
let $g_1={\hat g}\left(\boldsymbol{x},t;\boldsymbol{\Theta}_{G}\right)+{\hat g}_d\left(\boldsymbol{x},t;\boldsymbol{\Theta}_{G}\right)+{\hat f}_N\left(\boldsymbol{x},t;\boldsymbol{\Theta}_{U}\right)$, $g_2=g\left(\boldsymbol{x},t\right)$,
 there exists a finite Lipschitz constant $0<c \in \mathbb{R}$ satisfying
\begin{equation}\notag
\begin{aligned}
&\left\|\hat{u}\left(\boldsymbol{x},t;\boldsymbol{\Theta}_{U}\right)-u\left(\boldsymbol{x},t\right)\right\|\\
&\hspace{2mm}\leq c\left\|{\left({\hat g}\left(\boldsymbol{x},t;\boldsymbol{\Theta}_{G}\right)+{\hat g}_d\left(\boldsymbol{x},t;\boldsymbol{\Theta}_{G}\right)+{\hat f}_N\left(\boldsymbol{x},t;\boldsymbol{\Theta}_{U}\right)\right)}-{g}\left(\boldsymbol{x},t\right)\right\|\\
&\hspace{2mm}=c\left\|{\left(g\left(\boldsymbol{x},t\right)+{\hat f_N}\left(\boldsymbol{x},t;\boldsymbol{\Theta}_{U}\right)\right)}-{g}\left(\boldsymbol{x},t\right)\right\|\\
&\hspace{2mm}=c\left\|{\hat f}_N\left(\boldsymbol{x},t;\boldsymbol{\Theta}_{U}\right)\right\| .
\end{aligned}
\end{equation}
i.e.,
\begin{equation}
\label{eq:bound}
\left\|\hat e_N\right\| \leq c\left\|\hat f_N\right\|,
\end{equation}
which indicates that the  bound of CPINN approximation error   depends on physics-informed approximation
error.
According to the H{\"o}lder inequality, the  bound 
of $\left\|\hat{f}_N\right\|$  is given as follows:
\begin{equation}\notag
\begin{aligned}
\left\|{\hat f}_N\right\|&\leq\left[\int_{0}^{T} \int_{\bar\Omega}{\hat f}_N\left(\boldsymbol{x},t;\boldsymbol{\Theta}_{U}\right)^2 {\rm d} \boldsymbol{x}{\rm d } t\right]^{\frac{1}{2}}\left[\int_{0}^{T} \int_{\bar\Omega}1^2 {\rm d} \boldsymbol{x}{\rm d } t\right]^{\frac{1}{2}}\\
&=\left[T{\rm card}\left(\bar\Omega\right)\right]^{\frac{1}{2}}\left[\int_{0}^{T}\int_{\bar\Omega}{\hat f}_N\left(\boldsymbol{x},t;\boldsymbol{\Theta}_{U}\right)^2 {\rm d} \boldsymbol{x}{\rm d } t\right]^{\frac{1}{2}}\\
&\leq\left[T{\rm card}\left(\bar\Omega\right)\right]^{\frac{1}{2}}\left(L_{DN}+ L_{PN} \right)^{\frac{1}{2}}\\
&= L_{\hat U}^{\frac{1}{2}}.
\end{aligned}
\end{equation}
Consequently,
\begin{equation}\notag
\begin{aligned}
\left\|\hat{u}-u\right\| & \leq c\left\|{\hat f}_N \right\| \leq cL_{\hat U}^{\frac{1}{2}}.
\end{aligned}
\end{equation}
Finally, let
\begin{equation}
 \label{eq:deltaAndepsilon}
\delta=\frac{\varepsilon^2}{c^2},
\end{equation}
 for which $L_{\hat U}<\delta$ yields
$$
\|\hat{u}-u\| \leq c \delta^{\frac{1}{2}}=\varepsilon.
$$
\end{proof}{ Note that} $\varepsilon$ is  a monotonically increasing function about $\delta$ in \eqref{eq:deltaAndepsilon}, $\varepsilon\rightarrow 0$ means  $\delta\rightarrow 0$  and $\hat u \rightarrow u$.
Based on Theorem \ref{thm1}, the following {corollary} can be further obtained.
\newtheorem{corollary}{\bf Corollary}
\begin{corollary}\label{corollary}
Let ${\boldsymbol{\Theta}}_{U}\in \mathbb{R}^{p}$ and ${\boldsymbol{\Theta}}_{G}\in \mathbb{R}^{q}$ denote the parameters of $\hat{u} \left(\boldsymbol{x}, t;{\boldsymbol{\Theta}}_{U}\right)$ and $\hat{g} \left(\boldsymbol{x}, t;{\boldsymbol{\Theta}}_{G}\right)$, respectively. Assume that $\hat u:\Omega\times\left[0,T\right]\rightarrow\mathbb{R}$ and $\hat g:\Omega\times\left[0,T\right]\rightarrow\mathbb{R}$ are continuous functions with respect to ${\boldsymbol{\Theta}}_{U}$ and ${\boldsymbol{\Theta}}_{G}$, respectively. 
Then, for all $\delta >0$, there exists $\left\{{\boldsymbol{\Theta}}_{U},
{\boldsymbol{\Theta}}_{G}\right\}$\ $\in \mathbb{R}^{p}\times\mathbb{R}^{q}$ satisfying
$$L_{\hat U} \left({\boldsymbol{\Theta}}_{U},{\boldsymbol{\Theta}}_{G}\right)<\delta. $$
\end{corollary}
\begin{proof}
The exact source term to \eqref{eq:PaperGeneral1} is given by $g$.
 Theorem \ref{thm1} also implies that there exists a $\delta>0$ for all $\varepsilon>0$, such that
\begin{equation}\label{eq:lossfunction230210}
L_{\hat {U}}<\delta \Longrightarrow\left\|\hat{g}-g\right\|<\varepsilon.
\end{equation}Note that {$\hat u\left(\boldsymbol{x},t;\boldsymbol{\Theta}_{U}\right)$ and $\hat g\left(\boldsymbol{x},t;\boldsymbol{\Theta}_{G}\right)$} are continuous functions  with respect to ${\boldsymbol{\Theta}}_{U}$ and ${\boldsymbol{\Theta}}_{G}$, respectively.
Furthermore, range$(\hat u)=\mathbb{R}$ and range$(\hat g)=\mathbb{R}$. 
Then, for all $\frac{\varepsilon}{2}>0$, there exists $\boldsymbol{\Theta}_{U}$ and $\boldsymbol{\Theta}_{G}$,  it holds that
\begin{equation}
\begin{aligned}
\left\|\hat{u}-u\right\|<\frac{\varepsilon}{2},\quad \left\|\hat{g}-g \right\|<\frac{\varepsilon}{2}
\end{aligned}.
\end{equation}
According to the second power of $L^2$-norm loss function in \eqref{eq:lossU}, $ L_{\hat U}$ is a continuous function with respect to ${\boldsymbol{\Theta}}_{U}$ and ${\boldsymbol{\Theta}}_{G}$. Consequently,  the corollary holds.
\end{proof}
\newtheorem{remark}{\bf Remark}
\begin{remark}
Because $\mathbb{R}^{p\times q}$ is compact, Theorem \ref{thm1} and Corollary \ref{corollary} ensure that there exists a series of parameters $\left\{{\boldsymbol{\Theta}}_{U}^{(j)},{\boldsymbol{\Theta}}_{G}^{(j)}\right\}$ to realize $L_{\hat U} \left({\boldsymbol{\Theta}}_{U}^{(j)},{\boldsymbol{\Theta}}_{G}^{(j)}\right)\rightarrow 0\ (j\rightarrow\infty)$. 
That is, if CPINN is well-trained to achieve {$\left\{{\boldsymbol{\Theta}}_{U}^{(j)},{\boldsymbol{\Theta}}_{G}^{(j)}\right\}$} using iterative {optimization methods}, CPINN can approximate the exact solutions to the well-posed nonhomogeneous PDEs  well.
\vspace{-2mm}
\end{remark}
\subsection{Hierarchical Training Strategy for CPINN}
Owning to the practical operating environment limitations, the exact functional forms  or even sparse measurements of $g(\boldsymbol{x}, t)$ are unavailable. 
Considering the mutual dependence between $NetU$ and $NetG$  shown in \eqref{eq:bound}, a hierarchical training strategy is proposed for optimizing and couping $NetU$ and $NetG$ to iteratively estimate  ${\boldsymbol{\Theta}}_{U}$ and ${\boldsymbol{\Theta}}_{G}$.
Let $\boldsymbol{\hat\Theta}_{U}^{(k)}$ denote the estimated parameter of $NetU$ at $kth$ iteration step and
{$\boldsymbol{\hat\Theta}^{(k+1)}_{G}$ denote the estimated parameter of  $NetG$ at $(k+1)th$ iteration step}.
The core issue of the hierarchical training strategy is to solve the following two coupled optimization problems
\begin{equation}\label{eq:NetGloss}
\begin{aligned}    
\displaystyle
\boldsymbol{\hat\Theta}_{G}^{(k+1)}&=
\underset{{\boldsymbol{\Theta}}_{G}}{\arg \min }\hspace{1mm}
\left\{{{\rm MSE}_{DN}\left(\boldsymbol{\hat\Theta}^{(k)}_{U} \right)}+
{{\rm MSE}_{PN}\left({\boldsymbol{\Theta}}_{G} ; \boldsymbol{\hat\Theta}_{U}^{(k)}\right)}
\right\} \\[1mm]
&=\underset{{\boldsymbol{\Theta}}_{G}}{\arg \min }\hspace{2mm} {{\rm MSE}_{PN}\left({\boldsymbol{\Theta}}_{G} ; \boldsymbol{\hat\Theta}_{U}^{(k)}\right)}
\end{aligned}\vspace{-2mm}
\end{equation}
and
\begin{equation}\label{eq:NetUloss}
\hspace{-4mm}\boldsymbol{\hat \Theta}_{U}^{(k+1)}\hspace{-0.5mm}=\hspace{-0.5mm}
\underset{{\boldsymbol{\Theta}}_{U}}{\arg \min }
\left\{
{{\rm MSE}_{DN}\left({\boldsymbol{\Theta}}_{U} \right)}\hspace{-0.5mm}+\hspace{-0.5mm}
{{\rm MSE}_{PN}\left({\boldsymbol{\Theta}}_{U} ; \boldsymbol{\hat\Theta}^{(k+1)}_{G}\right)}
\right\}.
\end{equation}The details of the hierarchical training strategy are described in Algorithm~\ref{alg:algorithm-label}. The architecture of CPINN is shown in Fig. \ref{fig:CPINN}, in which the iterative transmissions of ${\boldsymbol{\Theta}}_{U}$ and ${\boldsymbol{\Theta}}_{G}$ happen.
\begin{algorithm}
    \caption{{The hierarchical training strategy of optimizing and coupling for CPINN.}}
    \label{alg:algorithm-label}
    \begin{algorithmic}
   \STATE  \textbf{Initialization} ($k=0$) \\
   -The training dataset $(\boldsymbol{x}, t, u)\in D$  and collocation points $(\boldsymbol{x}, t)\in E$ are obtained. 
   
-Randomly generate parameters ${\boldsymbol{\Theta}}^{(0)}_{U}$ and ${\boldsymbol{\Theta}}^{(0)}_{G}$  for initializing $NetU$ and $NetG$, respectively.
  \WHILE {Stop criterion is not satisfied}
    \STATE -Training for $NetG$ by solving the optimization problem~\eqref{eq:NetGloss} to obtain {$\boldsymbol{\hat\Theta}_{G}^{(k+1)}$},
    where the  estimation of $\hat u_{t} \left(\boldsymbol{x},t; \boldsymbol{\hat \Theta}_{U}^{(k)}\right)+
    \mathcal{N}\left(\hat u(\boldsymbol{x}, t;\boldsymbol{\hat \Theta}^{(k)}_{U}\right)$ in ${\rm MSE}_{PN}$ is obtained from the former iteration result
    $\boldsymbol{\hat\Theta}^{(k)}_{U}$.
    \STATE -Training for $NetU$ by solving the optimization problem \eqref{eq:NetUloss} to obtain
    $\boldsymbol{\hat\Theta}^{(k+1)}_{U}$.
      \STATE -$k=k+1$;
         \ENDWHILE
    \STATE  \textbf{Return} the output $u(\boldsymbol{x},t;\boldsymbol{\hat\Theta}_{U-\mathrm{CPINN}})$ of CPINN.
    \end{algorithmic}
\end{algorithm} 
\begin{figure}[ht]
\begin{center}
\includegraphics[width=7cm]{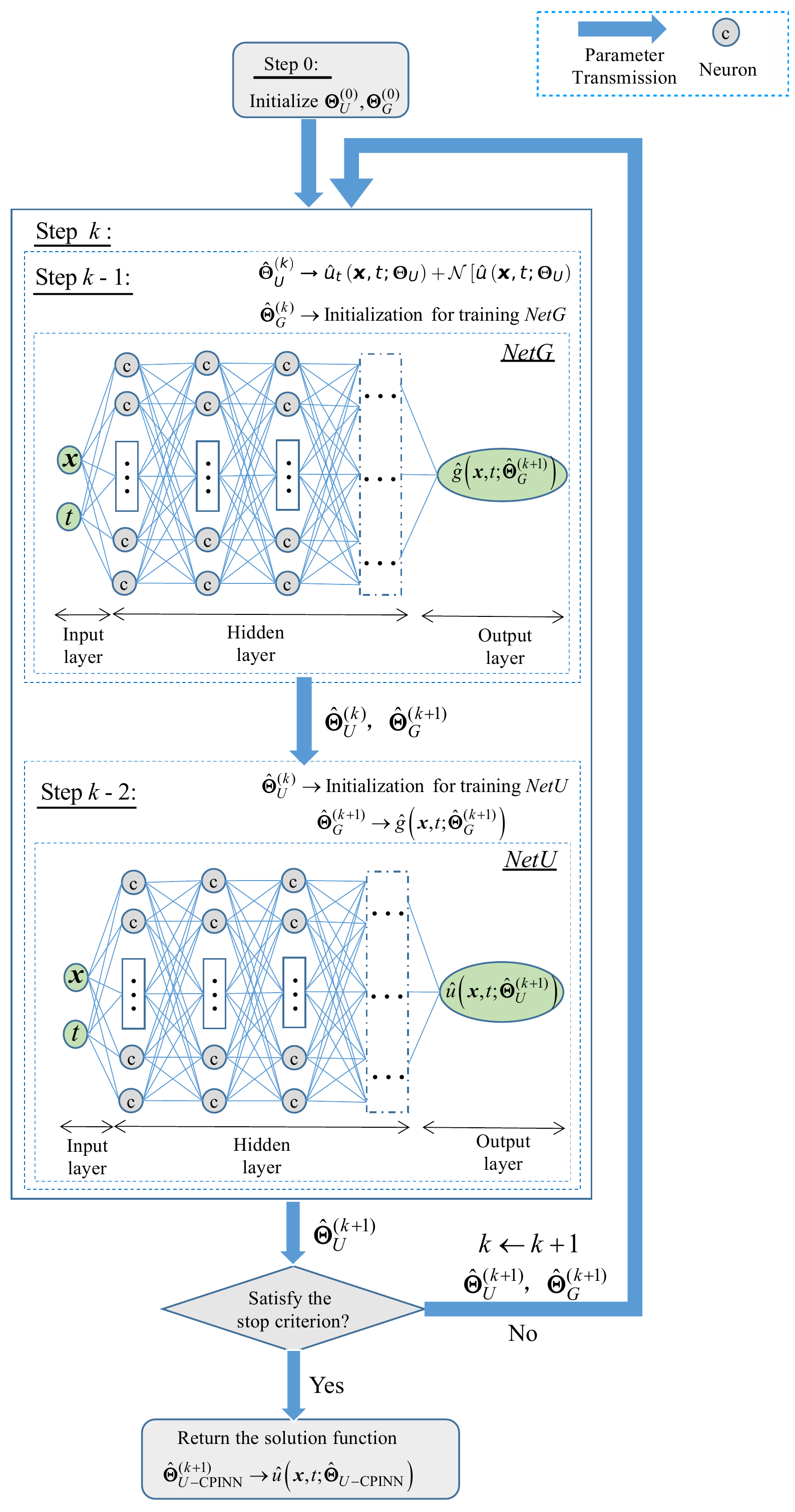}
\caption{Architecture of CPINN.}
\label{fig:CPINN}
\end{center}
\end{figure}
\subsection{CPINN-RP for Soft Sensors}

If PDEs include the partial derivative of temporal variables, the temporal dependence in the measurements can be  used to  improve the prediction performance.
Note that the mean squared error loss function \eqref{eq:loss} can be considered a Monte-Carlo approximation of the second power of  $L^2$-norm loss function \eqref{eq:lossU}.
The training datasets are sampled with a fixed temporal interval. 
This discretization strategy can cause information loss.
For this reason, CPINN is followed by RP to obtain CPINN-RP.  
In the first phase, CPINN composed of $NetU$ and $NetG$ is proposed to approximate PDEs solutions, and the hierarchical training strategy {\bf Algorithm} \ref{alg:algorithm-label}  optimizes and couples the two networks  to achieve parameter $\boldsymbol{\hat{\Theta}}_{U-\mathrm{C}\mathrm{PINN}}$.
In the second phase,  $NetU$ compensated with RP  obtaining $NetU$-$RP$ to improve the prediction performance,
which is initialized by $\boldsymbol{\hat\Theta}_{U}^{(k)}$. The spatial and temporal variables of  CPINN inputs are still fed in $NetU$-$RP$.
The rest inputs  of $NetU$-$RP$ are fed in an either-or way: 1) If a hard sensor is available at a collocation point, its delayed measurements  are used; 2) Otherwise, the delayed outputs of CPINN are used.
The training strategy for $NetU$-$RP$ is shown in Fig. \ref{fig:CPINN-RP}.
\begin{figure}[ht]
\begin{center}
\includegraphics[width=7.5cm]{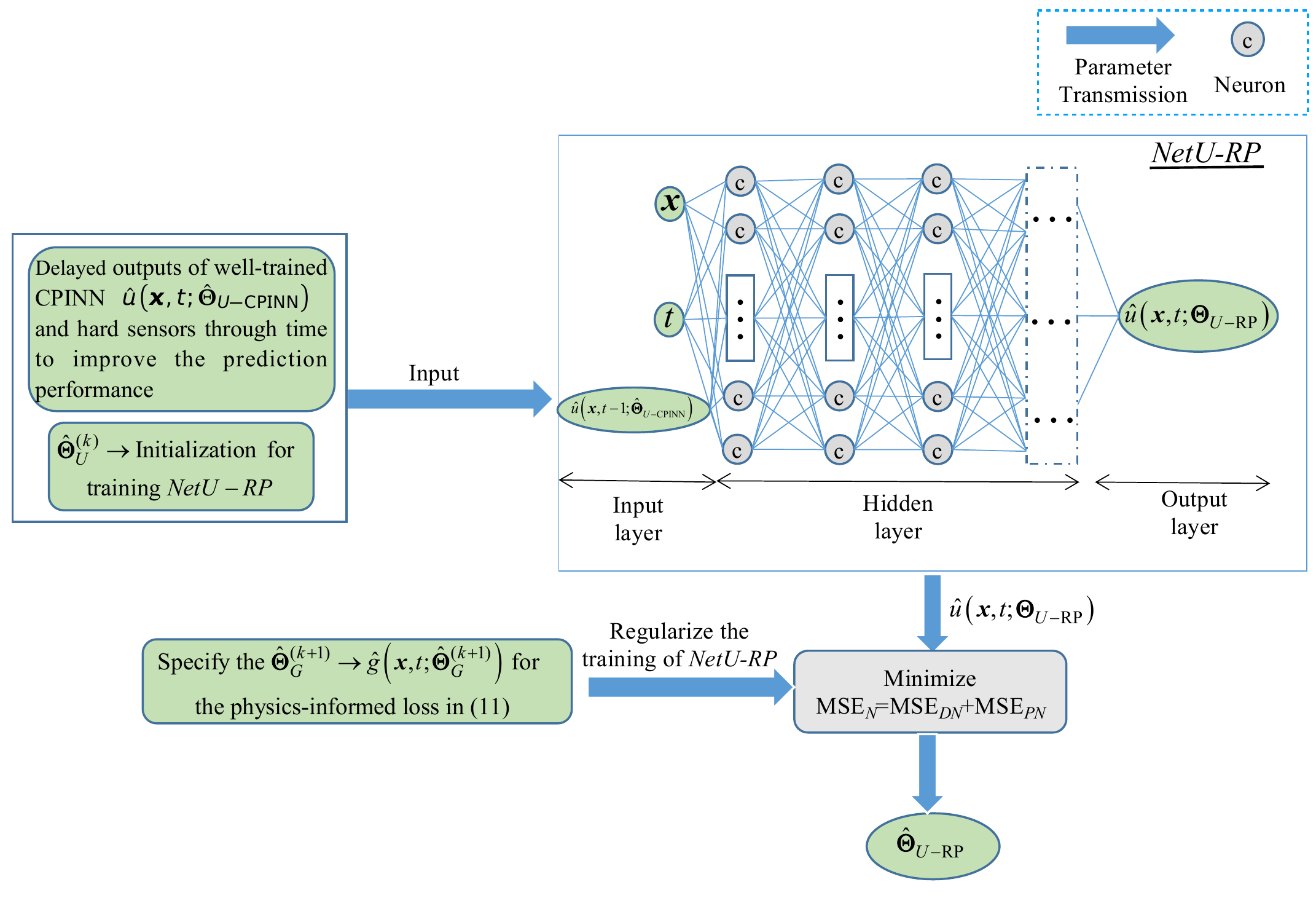} 
\caption{Training strategy for $NetU$-$RP$.}
\label{fig:CPINN-RP}
\end{center}
\end{figure}
The same loss function for training CPINN is also used to train $NetU$-$RP$ to achieve the final CPINN-RP output $u\left(\boldsymbol{x},t;\boldsymbol{\hat{\Theta}}_{U-\mathrm{RP}}\right)$ as soft sensing results.
\section{Numerical Experiments} 
In this section,
the performance of CPINN-RP is verified with numerical experiments implemented with Pytorch.
The fully-connected neural networks with hyperbolic tangent activation functions are used (ensuring  Theorem \ref{thm1} and Corollary \ref{corollary} hold) and  initialized by Xavier\cite{fang2021high}.
L-BFGS~\cite{zhu1997algorithm} is used  to  train  CPINN-RP.

We evaluate the prediction performance of  CPINN-RP  by means of root mean squared error (RMSE)
\begin{equation}\notag
\rm {RMSE}=\sqrt{\frac{1}{{\rm card}\left(T_e\right)}\sum_{(\boldsymbol{x},t,u)\in T_e} \left({\hat u\left(\boldsymbol{x},t\right)} -
 u\left(\boldsymbol{x}, t\right)\right)^{2}},
\end{equation}
where  $T_e$ is the testing dataset.
$\hat u\left(\boldsymbol{x}, t\right)$ and $ u\left(\boldsymbol{x}, t\right)$ denote the prediction and the corresponding ground truth, respectively.
To further validate the prediction performance of CPINN-RP, the Pearson correlation coefficient (CC)
\begin{equation}\notag
 {\rm CC} = \frac{{\rm Cov}\left(\hat u\left(\boldsymbol{x},t\right), u\left(\boldsymbol{x},t\right)\right)}{\sqrt{{\rm Var}   \left(\hat u\left(\boldsymbol{x},t\right)\right)}\sqrt{{\rm Var} \left(u\left(\boldsymbol{x}, t \right)\right)}}
\end{equation} 
is also used to measure the similarity between prediction and ground truth,
where ${\rm Cov}\left(\cdot,\cdot\right)$ is covariance and ${\rm Var}  (\cdot) $ is variance.
\vspace{-3mm}
\subsection{Case 1:  Heat Equation with Unmeasurable  Heat Source Term}
The  heat equation was first developed for  modeling  how heat diffuses through a given region\cite{bubanja2022mathematical, 8709973}.
 CPINN-RP is first verified with the following heat equation 
\begin{equation}\label{eq:HeatEquationNB}
\left\{
\begin{array}{ll}
\displaystyle
\frac{\partial u}{\partial t}=a^2 \frac{\partial^2 u}{\partial x^2}+g(x, t), \hspace{4mm}0<x<L, t>0 \\ [3mm]
\displaystyle
\left.u\right|_{t=0}=\phi(x), \quad \hspace{14mm} 0\leqslant x \leqslant L \\ [2mm]
\displaystyle
\left.u\right|_{x=0}=0,\left.\quad \frac{\partial u}{\partial x}\right|_{x=L}=0 ,\hspace{3mm}t>0
\end{array},
\right.
\end{equation}
where the thermal diffusivity coefficient $a=1$, the length of the finite line $L=\pi$,
the initial temperature $\phi(x)=\sin \left({x}/{2}\right)$,
and the  heat source  $ \displaystyle g(x)= \sin\left({x}/{2}\right)$. 
The exact solution $u(x,t)$  with respect to \eqref{eq:HeatEquationNB} is obtained according to \cite{kusse2010mathematical}.
The abovementioned setups are used to generate the training and testing datasets. Considering the practical situations, the  heat source is assumed to be unmeasurable.

In this case, $NetU$ consists of 3 hidden layers of 30 neurons individually; $NetG$ is of 8 hidden layers of 20 neurons individually; 
$NetU$-$RP$ is of the same hidden layer structure as $NetU$.
A total of 130 training data is randomly sampled from $ D_{B}$.
Moreover, the 20 sparse collocation points are randomly sampled to regularize the structure of~\eqref{eq:HeatEquationNB}.
The  abovementioned training dataset and the magnitude of  predictions $\hat u(x,t)$  are shown in Fig.~\ref{fig:241}(a).
Moreover, we compare the predictions and ground truths at fixed-time $t=3$ and 7 in Fig.~\ref{fig:241}(b) and (c), respectively.
Table~\ref{tb:241} shows the evaluation criteria for the prediction performance of CPINN-RP.
\begin{figure}[ht]
\begin{center}
\includegraphics[width=6cm]{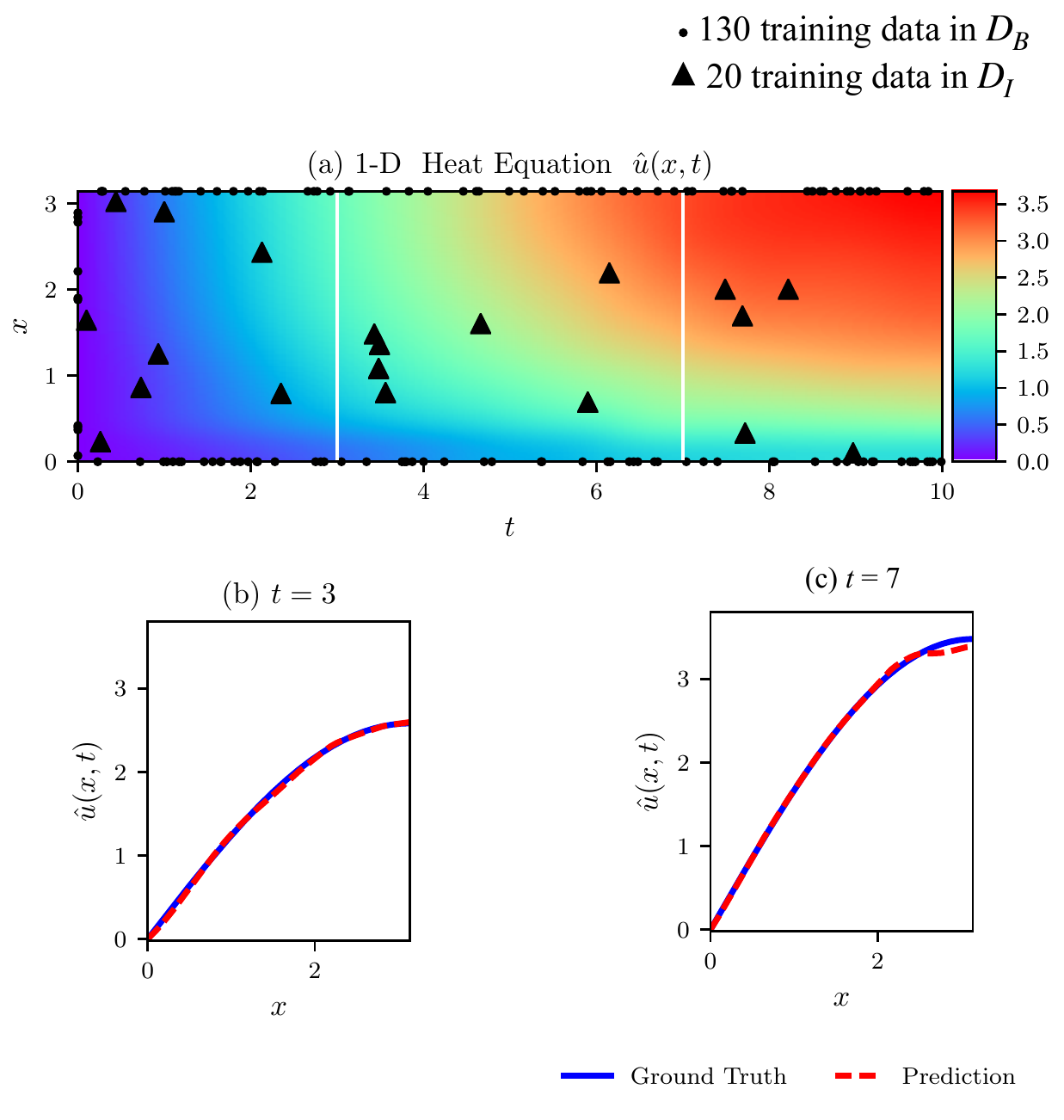}   
\caption{(a) Predictions $\hat u\left(x,t\right)$ for  1-D heat equation.
(b) and (c) Comparisons of  predictions and ground truths corresponding  to fixed-time $t$=3 and 7 snapshots depicted by  dashed vertical lines in (a), respectively.}
\label{fig:241}
\end{center}
\end{figure}
\begin{table}[H]
\begin{center}
\caption{Evaluation criteria for temporal snapshots are depicted by  dashed vertical lines in Fig.~\ref{fig:241}-(a).}
\label{tb:241}
\setlength{\tabcolsep}{1.5mm}{
\begin{tabular}{ccccccc}
\hline
\tabincell{c}{Criteria} &
\tabincell{c}{3} &
\tabincell{c}{7} &
\tabincell{c}{$\left[0,\pi\right]\times\left[0,10\right]$}
\\\hline
{\tabincell{c}{${\rm RMSE}$}}&   1.633733e-02 & 3.920911e-02& 4.234160e-02  \\
{\tabincell{c}{CC}} & 9.999875e-01 & 9.999571e-01& 9.999368e-01  \\\hline
\end{tabular}}
\end{center}
\end{table}
\subsection{Case 2: Wave Equation with Unmeasurable Driving Force Term}
The wave equation is a second-order linear PDE describing various fluctuation phenomena \cite{9062594,chu2021acoustic}.
In this section, we further verify CPINN-RP with the following wave equation
\begin{equation}\label{eq:WaveEquation}
\left\{
\begin{array}{ll}
\displaystyle
\frac{\partial^2 u}{\partial t^2}=a^2 \frac{\partial^2 u}{\partial x^2}+g(x, t) ,\hspace{4mm}0<x<L, t>0 \\[3mm]
\displaystyle
\left.u\right|_{t=0}=0,\left.\quad \frac{\partial u}{\partial t}\right|_{t=0}=0  ,\hspace{4mm}0 \leqslant x \leqslant L \\[3mm]
\displaystyle
\left.u\right|_{x=0}=0,\left.\quad u\right|_{x=L}=0  ,\hspace{6mm}t>0 
\end{array},
\right.
\end{equation}
where the velocity $a=1$,
the length of finite line $L=\pi$,
the time of  $t=6$,
and the force 
$g(x, t)=\sin \frac{2 \pi x}{L} \sin \frac{2 a \pi t}{L}$.
The acquisition of training and testing datasets is similar to Case 1.

In this experiment, $NetU$  consists of 3 hidden layers of 30 neurons individually; $NetG$ consists of 8 hidden layers of 20 units individually; $NetU$-$RP$ is of the same hidden layer structure as $NetU$.
A total of 210 training data in $D$,
170 training data in $D_{B}$  and 40 training data in $D_{I}$, are randomly sampled.
Fig.~\ref{fig:1DWaveEquation229}(a) shows the training dataset and the magnitude of predictions $\hat u(x,t)$.
Fig.~\ref{fig:1DWaveEquation229}(b) and (c) show the comparisons of the predictions and ground truths corresponding to the fixed-time $t$ = 2 and 4.
The prediction performance of CPINN-RP is further quantified in Table~\ref{tb:229}.
\begin{figure}[ht]
\begin{center}
\includegraphics[width=6.5cm]{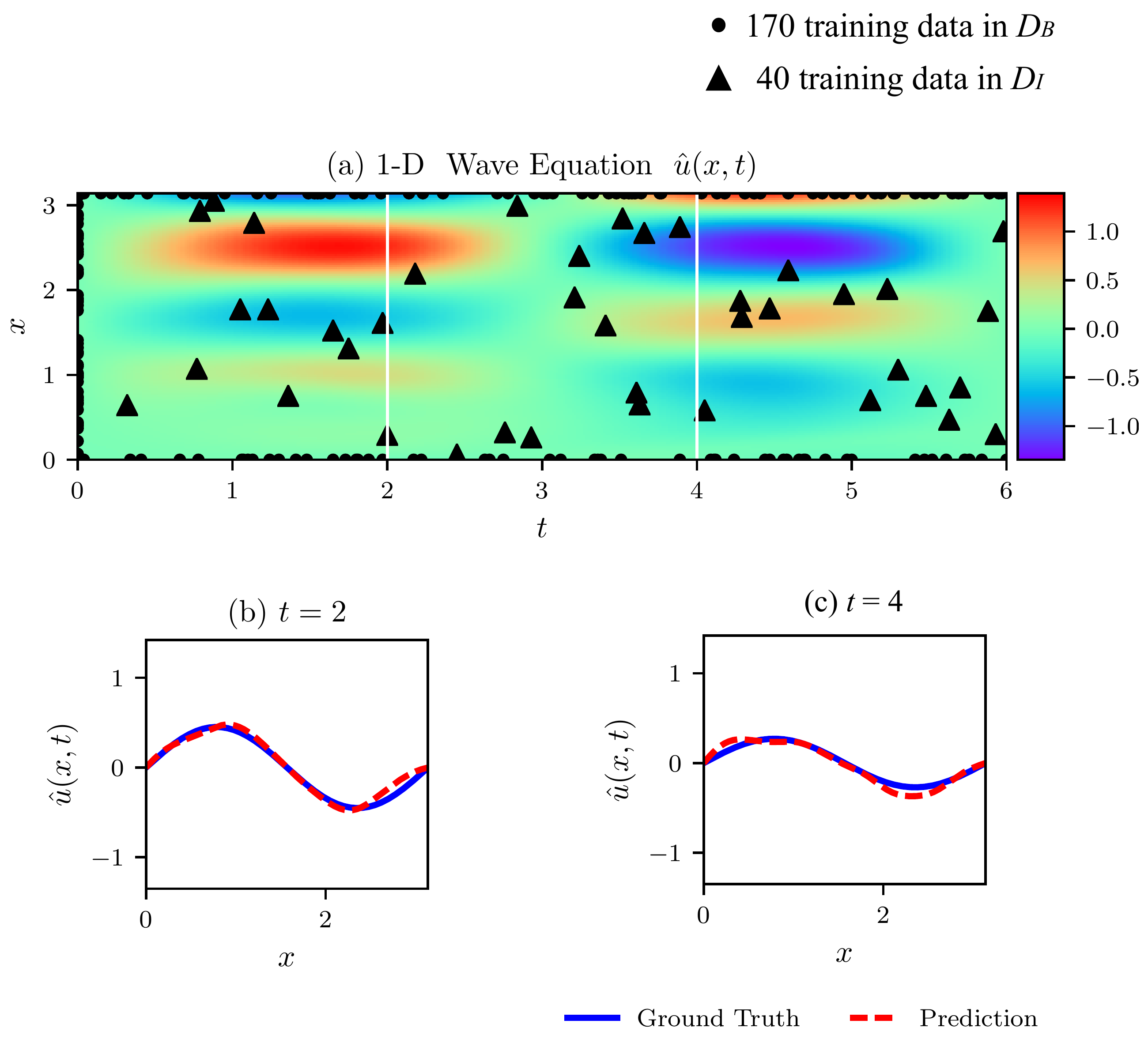}    
\caption{(a) Predictions $\hat u\left(x,t\right)$ for 1-D wave equation.
(b) and (c) Comparisons of predictions and ground truths corresponding to fixed-time $t$=2 and 4 snapshots depicted by  dashed vertical lines in (a), respectively.}
\label{fig:1DWaveEquation229}
\end{center}
\end{figure}
\begin{table}[ht]
\begin{center}
\caption{{Evaluation criteria for temporal snapshots {are} depicted by  dashed vertical lines in Fig.~\ref{fig:1DWaveEquation229}-(a).}}
\label{tb:229}
\setlength{\tabcolsep}{1.5mm}{
\begin{tabular}{ccccccc}
\hline
\tabincell{c}{Criteria} &
\tabincell{c}{2} &
\tabincell{c}{4} &\tabincell{c}{$\left[0,\pi\right]\times\left[0,6\right]$}\\\hline
{\tabincell{c}{${\rm RMSE}$}} & 4.844044e-02 & 5.280401e-02& 6.748852e-02 \\
 {\tabincell{c}{CC}} &9.898351e-01 &9.819893e-01& 9.876968e-01 \\\hline
\end{tabular}}
\end{center}
\end{table}
\section{Experimental Verification with  Data from a Practical Vibration Process} 
In this section,  
the practical datasets sampled from aero-engine involute spline couplings fretting wear experiment platform are used to demonstrate the feasibility and effectiveness of CPINN-RP for soft sensors.
The experiment system is presented in Fig. \ref{fig:platform} and Fig. \ref{fig:system}. 
\begin{figure}[ht]
\begin{center}
\includegraphics[width=7.5cm]{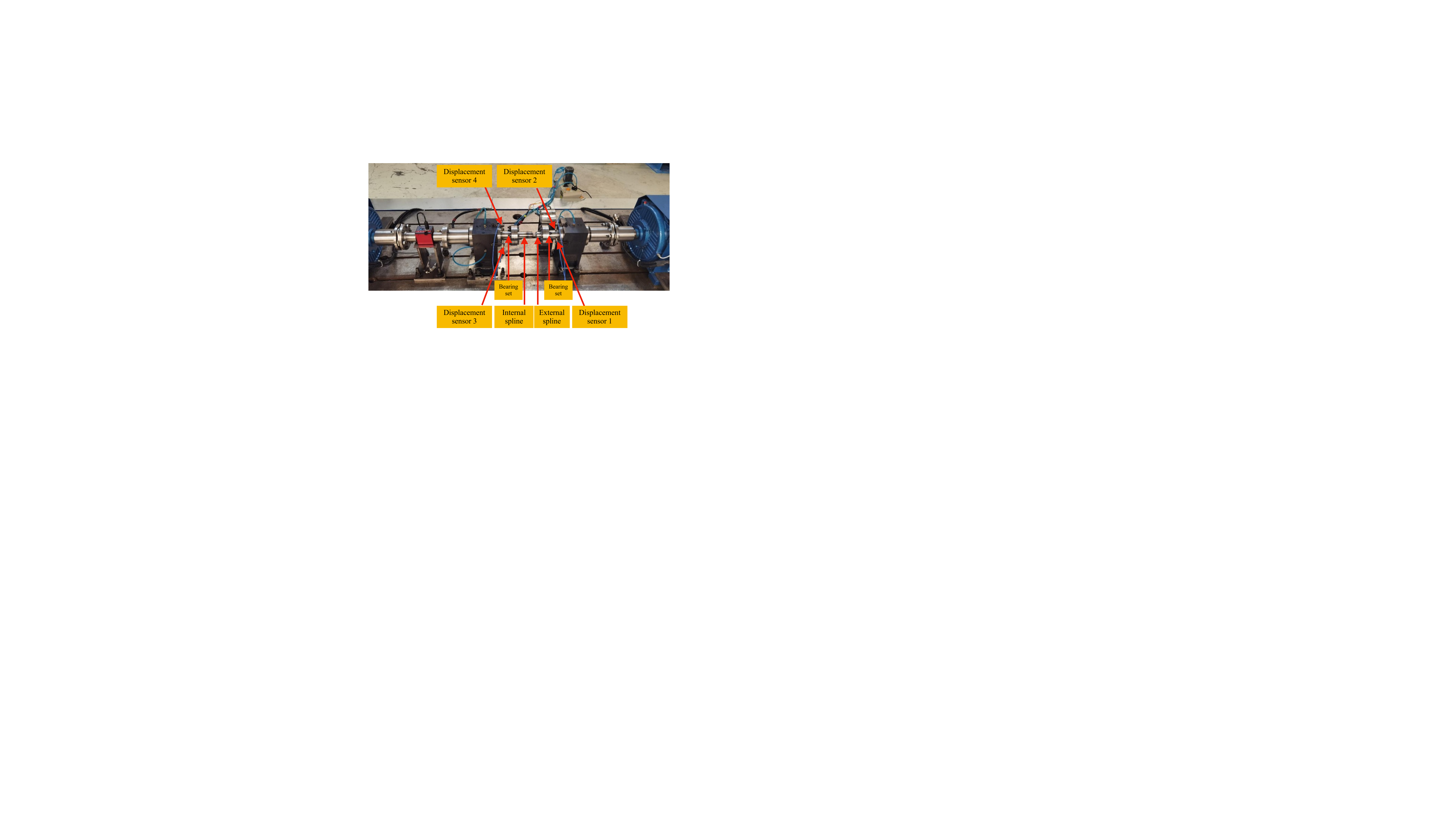}   
\caption{Aero-engine involute spline couplings fretting wear experiment platform.}
\label{fig:platform}
\end{center}
\end{figure}
\begin{figure}[ht]
\begin{center}
\includegraphics[width=7.5cm]{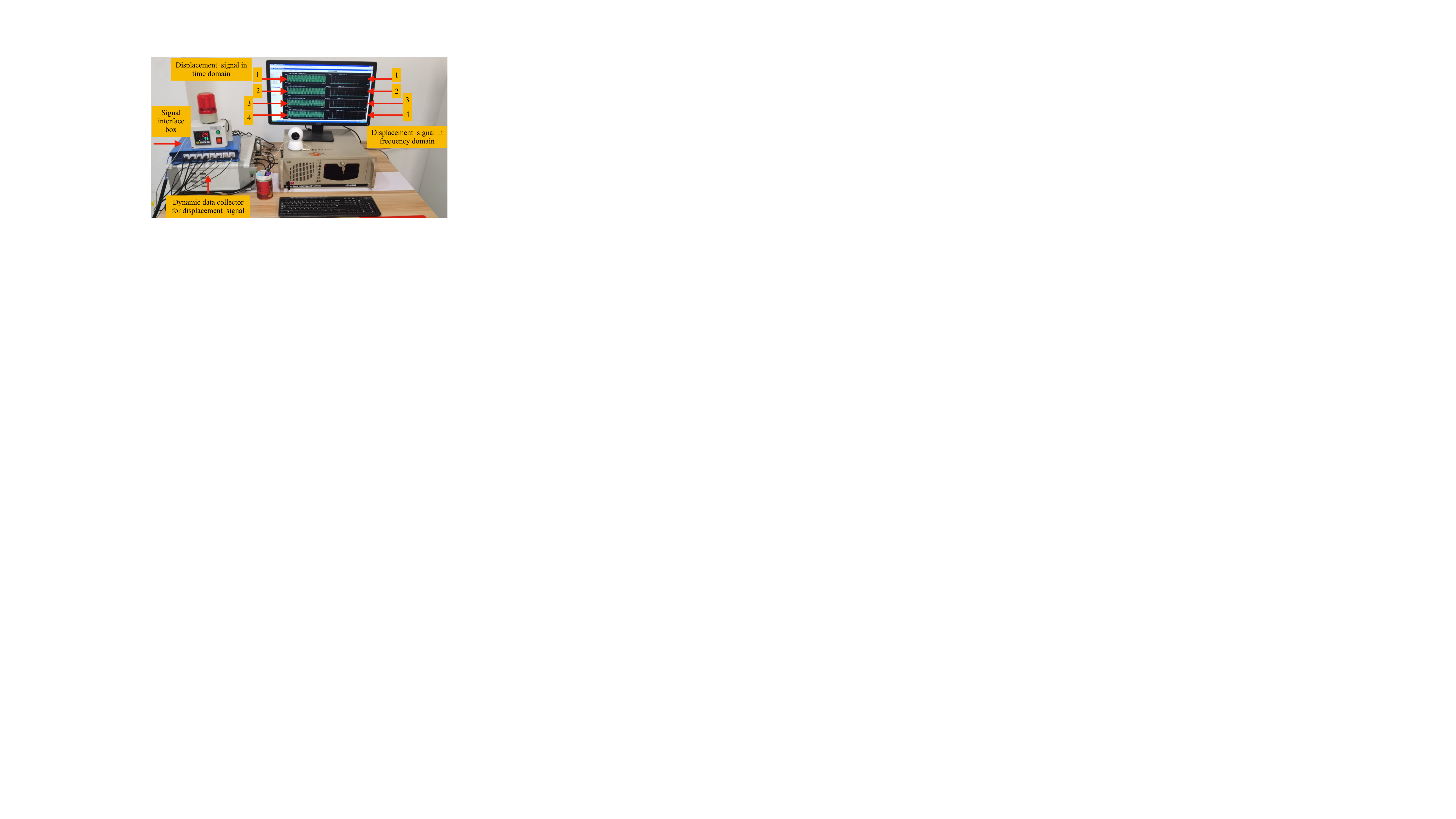}   
\caption{Data acquisition and processing system.}
\label{fig:system}
\end{center}
\end{figure}As shown in Fig. \ref{fig:platform}, the experiment platform is mainly composed of spline couplings, bearing sets,  displacement sensors, and a motor drive.
The data acquisition and processing system in Fig. \ref{fig:system} is used to monitor the vibration displacement of the internal and external spline shafts.
The datasets are sampled with the condition as follows.
\begin{itemize}
\item[1)] The working speed of the motor drive is 3000 r/min.
\item[2)] The sampling frequency is 2048Hz with 4096 sampling points.
\end{itemize}
The time-domain waveforms for the sampled signals are shown in Fig.~\ref{fig:signalTimeFrequencydomain230601}. 
\begin{figure}[ht]
\begin{center}
\includegraphics[width=7.5cm]{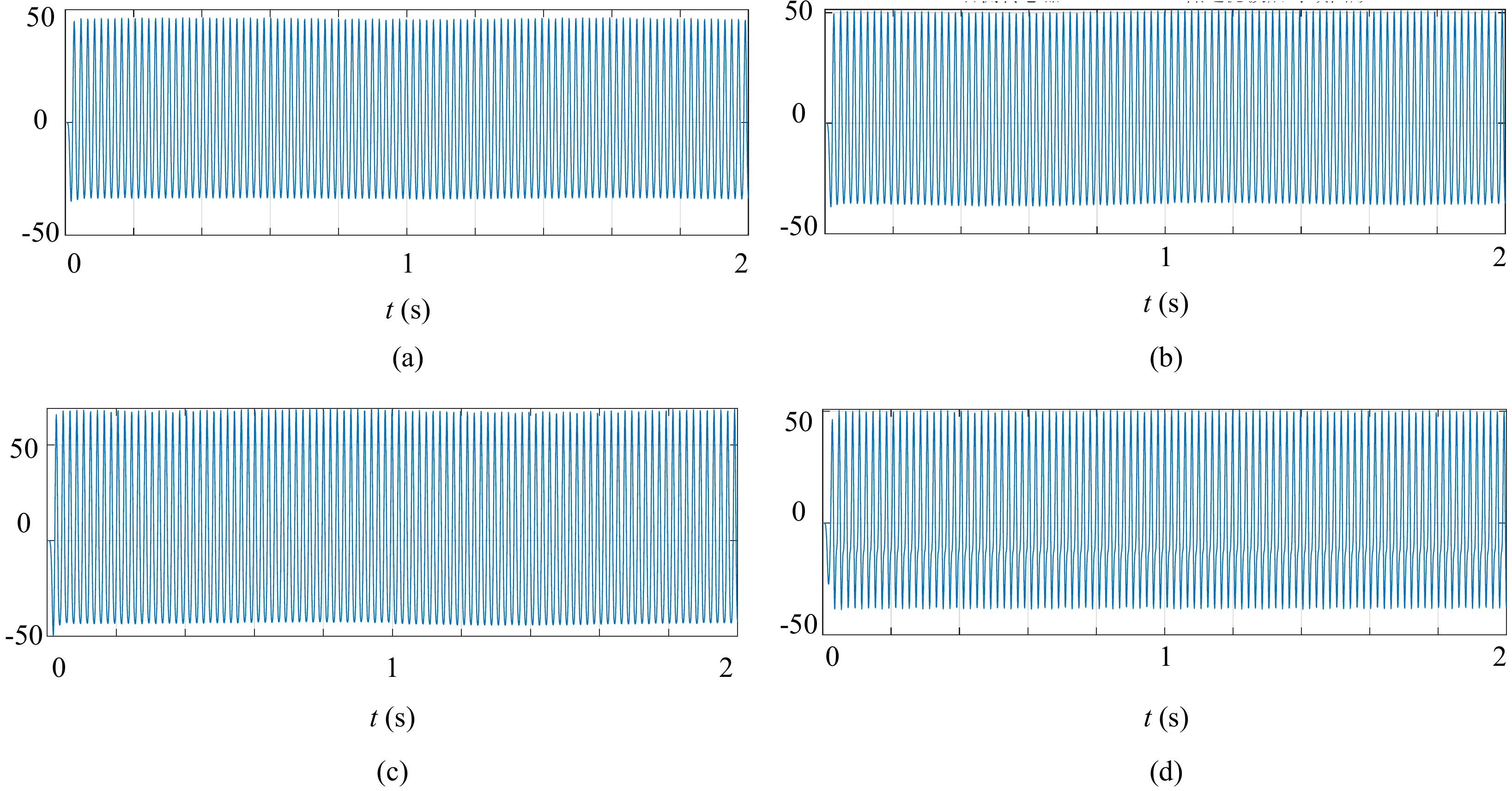}    
\caption{Raw displacement signals from sensors in time domain: (a) Sensor 1.
(b) Sensor 2.
(c) Sensor 3.
(d) Sensor 4.
}
\label{fig:signalTimeFrequencydomain230601}
\end{center}
\end{figure}
 As described in Section \Rmnum{4}, the vibration displacement of the spline shaft is  governed by \eqref{eq:WaveEquation}.
 The distance $L$ between the shaft of two ends is $520$mm
and measurements $u(x_i,t)$ for $i=1,2,3,4$ are
obtained from displacement sensors 1-4 as described in Fig. \ref{fig:platform} and \ref{fig:system}. 
The total number of training data is about 1$\%$ of the total sampled data.
The measurements obtained from  sensor 4 are assumed to be unavailable and CPINN-RP is used as a soft sensor for predicting the output of sensor 4.
 The setups of $NetU$, $NetG$, and $NetU$-$RP$
 can be referred to Section \Rmnum{4}.
The prediction performance of CPINN-RP is  depicted in Table~\ref{tb:platform}. 
\begin{table}[ht]
\begin{center}
\caption{Evaluation criteria for predictions in aero-engine involute spline couplings fretting wear experiment platform}
\label{tb:platform}
\setlength{\tabcolsep}{0.7mm}{
\begin{tabular}{cccccc}
\hline\hspace{-1mm}{Sensor} &
\tabincell{c}{Criteria} &
\tabincell{c}{Training} &
\tabincell{c}{Testing} &\tabincell{c}{$\left[0,5.2\right]\times\left[0,2\right]$}\\\hline
\hspace{-2mm}\multirow{2}{*}{1} &
{RMSE} &2.613778e-02 & 5.776048e+00& 5.772521e+00\\
&{\tabincell{c}{CC}} &9.999967e-01 & 9.879768e-01& 9.879771e-01 \\\hline

\hspace{-2mm}\multirow{2}{*}{2} &{RMSE} &1.088167e-02 & 5.449913e+00& 5.446586e+00\\
&{\tabincell{c}{CC}} &9.999997e-01 & 9.909315e-01& 9.909329e-01 \\\hline

\hspace{-2mm}\multirow{2}{*}{3} &
{RMSE} &9.480886e-03 & 4.762897e+00& 4.759989e+00\\
&{\tabincell{c}{CC}} &9.999998e-01 & 9.933763e-01& 9.933794e-01 \\\hline
\hspace{-2mm}\multirow{2}{*}{4} 
&{RMSE} &2.144921e+00 & 8.222315e+00& 8.217437e+00\\ 
&{\tabincell{c}{CC}} &9.997592e-01 & 9.714094e-01& 9.714027e-01 
 \\\hline\end{tabular}}
\end{center}
\end{table}
Because the  measurements of sensor 4 are unavailable, the soft sensing results for  sensor 4 are not as good as the  other three sensors,  as shown in Table~\ref{tb:platform}. 
Although it does not outperform the hard sensors, CPINN-RP still offers feasible  spatiotemporal predictions for difficult-to-measure variables and can be expected to as a core model for soft sensors.

\section{Conclusion}
This article proposed CPINN-RP for soft sensors.
First, CPINN composed of  $NetU$ and $NetG$  is proposed to approximate the solutions to  nonhomogeneous PDEs.   
The proposed CPINN is theoretically proved to have a satisfying approximation capability for the solutions to the well-posed PDEs  with unmeasurable source terms.
Besides the theoretical aspects, $NetU$ and $NetG$ are optimized and coupled  by the proposed hierarchical training strategy.
Then, 
$NetU$-$RP$ is obtained with the recurrently delayed outputs of well-trained CPINN   and  hard sensors to further improve  the prediction performance.
Finally, the feasibility and effectiveness of CPINN-RP for soft sensors are validated with the artificial and practical datasets. 
Our method can be expected to benefit the  soft sensors for  industrial processes with spatiotemporal dependence and unmeasurable  sources.

In the future,
we will continue to use  CPINN-RP as a soft sensor for more industrial scenarios other than the vibration processes.
Meanwhile, more complex situations,
such as  PDEs with exactly unknown structures,  will be considered in CPINN-RP.
Feature extraction layers, such as convolution and pooling, will be added to CPINN-RP for architecture extensions.
\bibliographystyle{IEEEtran}
\bibliography{myref}
\if false
\vspace{-1cm}
\begin{IEEEbiography}[{\includegraphics[width=1in,height=1.25in,clip,keepaspectratio]{Author1}}]{Aina Wang}
received the B.S. degree in Electrical Engineering and Automation, in 2016, from Shenyang University of Technology, Liaoning, China. 
She is currently working toward the Ph.D. degree in electronics and information  with the School of Control Science and Engineering,
Dalian University of Technology, Liaoning, China. Her
research interests include optimization machine learning and optimization.
\end{IEEEbiography}
\vspace{-1cm}
\begin{IEEEbiography}[{\includegraphics[width=1in,height=1.25in,clip,keepaspectratio]{Author2}}]{Pan Qin}
received the B.S. and the M.S. in Electrical Engineering and Automation from Northwest Polytechnic University, Shaanxi, China, in 2000 and 2004, respectively. Then he received the Ph.D. in Information Science from Kyushu University, Fukuoka, Japan, in 2008. He worked as a research fellow for nearly six years in the School of Mathematics and Science and the Institute of Math-for-industry at Kyushu University.    
He is an associate professor at Dalian University of Technology, Dalian, China.    
His current research interests focus on statistics, machine learning, and optimization.\end{IEEEbiography}
\vspace{-1cm}
\begin{IEEEbiography}[{\includegraphics[width=1in,height=1.25in,clip,keepaspectratio]{Author_Xi-Ming Sun}}]{Xi-Ming Sun}
 (M'07-SM'13) received the Ph.D. degree in Control Theory and Control Engineering from the Northeastern University, China, in 2006. From August 2006 to December 2008, he worked as a Research Fellow in the Faculty of Advanced Technology, University of Glamorgan, UK. He then visited the School of Electrical $\&$ Electronic Engineering, Melbourne University, Australia in 2009, and Polytechnic Institute of New York University in 2011, respectively. He is currently a Professor in the School of Control Science and Engineering, Dalian University of Technology, China. He is IEEE Senior Member, member of IFAC and the associate editor of the journal of IEEE Transactions on Cybernetics. He was awarded the Most Cited Article 2006–2010 from the journal of Automatica in 2011. His research interests include switched delay systems, networked control systems, and nonlinear systems.
\end{IEEEbiography}
 \fi

\end{document}